\documentclass[11pt]{article}

\usepackage{graphics,color,graphicx,subfigure}
\usepackage{latexsym,amsfonts,amsmath,tabularx,epsfig,enumerate,epstopdf}
\usepackage{fullpage}

\def\bfx{{\bf x}}

\def\bfv{{\bf v}}
\def\bff{{\bf f}}
\def\nag{\nabla_{\Gamma}}
\def\delg{\Delta_{\Gamma}}
\def\S{\mathbb{S}}
\def\R{\mathbb{R}}

\def\calN{\mathcal{N}}
\def\bfn{\mbox{\boldmath$n$}}
\def\Ga{\Gamma}
\def\ol{\overline}
\def\pa{\partial}
\newcommand{\innp}[2]{\left\langle#1,#2\right\rangle}

\newtheorem{theorem}{Theorem}[section]
\newtheorem{lemma}{Lemma}[section]
\newtheorem{remark}{Remark}[section]
\newtheorem{definition}{Definition}[section]

\title{A Physics-Informed Neural Network Framework For Partial Differential Equations on 3D Surfaces: Time-Dependent Problems}

\author{
	Zhiwei Fang \thanks{University of Nevada, Las Vegas, 4505 S Maryland Pkwy, Las Vegas, NV 89154, USA ({\tt fangz1@unlv.nevada.edu}),}
	\and
	Justin Zhan \thanks{Department of Computer Science and Computer Engineering, University of Arkansas, Fayetteville, AR 72701, USA({\tt jzhan@uark.edu}),}
  \and
	Xiu Yang \thanks{Department of Industrial and Systems Engineering, Lehigh University, Bethlehem, PA 18015, USA ({\tt xiy518@lehigh.edu}).}
}
\date{}

\begin{document}
\maketitle
\begin{abstract}
In this paper, we show a physics-informed neural network solver for the
  time-dependent surface PDEs. Unlike the traditional numerical solver, no
  extension of PDE and mesh on the surface is needed. We show a simplified prior
  estimate of the surface differential operators so that PINN's loss value will
  be an indicator of the residue of the surface PDEs. Numerical experiments
  verify efficacy of our algorithm.
\end{abstract}

\section{Introduction}
Partial differential equations (PDEs) on manifolds has been widely used in
various areas. Especially, its application in image processing include mapping
an image on a given surface~\cite{turk1991}, recovering lost 
information~\cite{bertozzi2001}, and segmenting and deciphering 
images~\cite{tian2009,biddle2013}. Such techniques are widely used in biological and 
medical sciences, e.g., simulating animal coats~\cite{murray2003}, wound 
healing~\cite{olsen1998}, brain wrapping~\cite{toga1998}, lipid interactions in
membranes~\cite{elliott2010biomem}, and fluids in lungs~\cite{halpern1998}.

%Therefore, more and more mathematicians devoted themselves to the numerical methods for surface PDEs. 
Different numerical methods have been developed to solve PDEs on
surfaces~\cite{dziuk2013finite}.
In 1988, G. Dziuk~\cite{dziuk1988fem} established the finite element scheme for 
the surface PDEs, including the weak formulation and finite element spaces. 
Then, an extension was introduced in~\cite{bertalmio2000}.
%In \cite{bertalmio2000}, the authors showed an extension method to solve surface PDEs.
However, this method needs to extend both the PDEs and their solution to a 
neighborhood of the surface. 
%A review of finite element methods for surface PDEs can be found in~\cite{dziuk2013finite}. In recent years, 
A. Petras et al.~\cite{petras2016,petras2018rbffd,petras2019} report many works
on using the radial basis function finite difference (RBF-FD) method to solve
surface PDEs. This method is more efficient and easier to implement than finite
element methods. But one needs to adjust the shape parameters of radial basis
functions to balance the stability and accuracy.

Rather than the traditional numerical solvers, machine learning approaches 
attract more and more attentions.
%researchers and arouse interest in industries. 
In particular, physics-informed neural network (PINN)~\cite{pinn} has become a
popular method for solving PDEs and inverse problems. %From that to now, more than hundreds of papers reports its variations and applications. 
Further, the combination of PINN and adversarial networks can be
used for uncertainty quantification (UQ) problems in PDE~\cite{yang2019gan}. 
The finite element methods have been combined with the PINN to enhance its 
performance \cite{vpinn,hpvpinn}. 
One can even solve the PDEs without the concrete form of it as proposed in~\cite{raissi2018}.
NVIDIA also develops a scalable PINN solver
SIMNET~\cite{nvidia} 
%In~\cite{fang2019meta}, the PINN has been used to design the electromagnetic metamaterial. 
%(\texttt{https://developer.nvidia.com/simnet}), which is the PINN %solvers for PDEs.

In this work, we will proceed with PINN's work for surface PDEs~\cite{fang2019spde} and develop the PINN algorithm for time-dependent PDEs on the surface. We will show a simplified proof about a prior estimate of the surface differential operator. Continuous-time and discrete-time method have been studied for the surface PDEs. The numerical examples will verify our algorithm.

This paper is organized as follows. In section \ref{sec:spde}, we briefly introduce the mathematical preliminary of surface PDEs. We will show a prior estimate of the surface differential operators so that PINN's loss value will serve as an indicator of surface PDE's residue. PINNs for traditional PDEs and our algorithm for surface PDEs have been shown in section \ref{sec:pinn}. Then, a variety of numerical experiments have been shown to verify our algorithm in section \ref{sec:num}. We conclude this paper in section \ref{sec:con}.

\section{Partial Differential Equations on Surfaces}\label{sec:spde}
In this section, we briefly introduce the elementary definitions of PDEs on
surfaces. For more details about calculus on the manifolds, we refer interested
readers to~\cite{dziuk2013finite}.
\subsection{Differential operators on surfaces}
Let $ \Ga $ be a smooth surface embedded in $ \R^3 $ with unit normal $ \bfn $,
and $ u:\Ga\mapsto\R $ be a function on $ \Gamma $. Denoted by $ \ol{u} $ the
smooth extension of $ u:\Ga\mapsto\R $ to $ \ol{u}:U\mapsto\R $, where $ U $ is
a neighborhood of $ \Ga $, such that $ \ol{u}|_\Ga=u $ and $ \ol{u} $ is a
constant along $ \bfn $ for each point on $ \Ga $. We use $ \nabla $, $
\nabla\cdot $ and $ \Delta $ to denote the ordinary gradient, the divergence, and the Laplace operator in $ \R^3 $, respectively. Now, we can define the gradient, the divergence, and the Laplace operator on $ \Ga $ for $ u $, which are essential for PDE on the surface.
\begin{definition}
	Let $ u:\Ga\mapsto\R $ has a continuous derivative (of class $ C^1 $), and then the gradient of $ u $ on $ \Ga $ is defined as:
	\[
	\nag u = \nabla \ol{u}-\innp{\nabla\ol{u}}{\bfn}\bfn,
	\]
	where $ \innp{\cdot}{\cdot} $ means the inner product in $ \R^3 $.
\end{definition}
\begin{definition}
	Let $ u:\Ga\mapsto\R $ has a continuous derivative (of class $ C^1 $), and
	\[
	\nag u = (D_1u,D_2u,D_3u).
	\]
	Then, the divergence of vector $ \bfv=(v_1,v_2,v_3)\in(C^1)^3 $ on $ \Ga $ is defined by:
	\[
	\nag\cdot \bfv = D_1 v_1 + D_2 v_2 + D_3 v_3.
	\]
\end{definition}
%{\color{red}We point out that in \cite{fang2019spde}, on page 26329, there is a typo of the definition of $ \nag\cdot $. We correct it here.}
\begin{definition}
	Let $ u:\Ga\mapsto\R $ have a second-order continuous derivative (of class $ C^2 $), then the Laplace operator on $ \Ga $ is defined as:
	\[
	\delg u = \nag\cdot\nag u.
	\]
	The Laplace operator on a surface is also known as the Laplace-Beltrami operator.
\end{definition}
\subsection{The equivalence principles}
Instead of solving surface the PDEs directly, we solve the ordinary PDEs on the surface with constraints. The idea of our method comes from the following lemma \cite{petras2018rbffd}:
\begin{lemma}\label{lem_equi}
	Let $ u $ be any function on $ \R^n $ that is constant along with normal
  directions of $ \Ga $. Then, at the surface, intrinsic gradients are
  equivalent to standard gradients:
	\[
	\nag u = \nabla u.
	\]
	Let $ \bfv $ be any vector field on $ \R^n $ that is tangent to $ \Ga $ and tangent to all surfaces displaced by a fixed distance from $ \Ga $. Then, at the surface,
	\[
	\nag\cdot \bfv = \nabla\cdot \bfv.
	\]
\end{lemma}
Therefore, by using this lemma, we may replace the surface gradient $ \nag u $
with $ \nabla u $ under the constraint $ \innp{\nabla u}{\bfn}=0 $. Then, the Laplace-Beltrami operator $ \delg u $ can be replaced by $ \Delta u $ with constraint $ \innp{\nabla u}{\bfn}=0 $. This is because under the constraint $ \innp{\nabla u}{\bfn}=0 $, $ \nabla u $ is a vector field tangent to $ \Ga $. Hence, 
\[
\delg u = \nag \cdot\nag u=\nag\cdot\nabla u=\nabla\cdot\nabla u =\Delta u,
\]
if we extend $ u $ as a constant along the $ \bfn $.

The original idea of PINNs is to set up a loss function by using the residue of
the PDE \cite{pinn}. When the loss function approaches zero, the residue of the
PDE goes to zero, which implies taht the PDE is satisfied on training points
approximately. However, we will show that this is not true for surface PDEs. 
Therefore, we develop a new algorithm in~\cite{fang2019spde} to solve this
problem. The following prior estimate will be used to improve our algorithm in
the next sections.
\begin{theorem}\label{thm}
	Let $ \Ga\subset \R^3 $ be a smooth surface whose mean curvature $ |H|<\infty $, $ \bfv $ and $ \bff $ be two vector fields on $ \Ga $, and $ u $ and $ g $ are two scalar functions on $ \Ga $, then we have
	\begin{eqnarray}
	\|\nag u-\bff\|	&\leq& \|\nabla \ol{u} - \ol{\bff}\|+\|\innp{\bfn}{\nabla \ol{u}}\|,\label{est_nag}\\
	\|\nag \cdot\bfv-g\|	&\leq& \|\nabla \cdot\ol{\bfv} - \ol{g}\|+\|\bfn^\top \nabla \ol{\bfv}\bfn\|,\label{est_div}\\
	\|\delg u - g\|	&\leq& C(\|\Delta \ol{u} -\ol{g}\|+\|\innp{\bfn}{\nabla \ol{u}}\|+\|\bfn^\top\nabla^2 \ol{u} \bfn\|)\label{est_del}
	\end{eqnarray}
	on $ \Ga $, where $ C=\max(1,2|H|) $ is a constant, column vector $ \bfn $ is the unit normal on $ \Ga $, $ \|\cdot\| $ is the $ L^2 $ norm for the scalar or the vector functions, and $ \nabla^2 $ is the Hessian operator. The overline means the $ \R^3 $ smooth extension as mentioned above.
\end{theorem} 
\begin{proof}
	\begin{enumerate}
		\item 
		By using the definition of $ \nag $, we have
		\begin{align*}
		\|\nag u-\bff\|	&=\|\nabla \ol{u} - \bfn\innp{\bfn}{\nabla \ol{u}}-\ol{\bff}\|\\
		&\leq \|\nabla \ol{u} -\ol{\bff}\|+\|\bfn\innp{\bfn}{\nabla \ol{u}}\|\\
		&\leq \|\nabla \ol{u} -\ol{\bff}\|+\|\innp{\bfn}{\nabla \ol{u}}\|.
		\end{align*}
		This proves (\ref{est_nag}).
		\item 
		By using the definition of $ \nag \cdot $, we have
		\begin{align*}
		\|\nag \cdot\ol{\bfv}-\ol{g}\|	&=\|\nabla\cdot\ol{\bfv} - \bfn^\top \nabla\ol{\bfv}\bfn - \ol{g}\|\\
		&\leq \|\nabla\cdot\ol{\bfv} - \ol{g}\|+\|\bfn^\top \nabla\ol{\bfv}\bfn\|.
		\end{align*}
		This proves (\ref{est_div}).
		\item
		By using the definition of $ \delg $, we have
		\begin{align*}
		\delg u - g	&=\nag\cdot\nag u - g\\
		&=\nabla\cdot(\nabla \ol{u}-\bfn\innp{\bfn}{\nabla \ol{u}})-\bfn^\top \nabla(\nabla \ol{u}-\bfn\innp{\bfn}{\nabla \ol{u}})\bfn-\ol{g}\\
		&=\Delta \ol{u} -\nabla\cdot(\bfn\innp{\bfn}{\nabla \ol{u}})-\bfn^\top \nabla^2 \ol{u} \bfn+\bfn^\top \nabla(\bfn\innp{\bfn}{\nabla \ol{u}})\bfn-\ol{g}.
		\end{align*}
		Note that, by the vector product rule, and $ \nabla\cdot\bfn =-2H $,
		\begin{align*}
		\nabla\cdot(\bfn\innp{\bfn}{\nabla \ol{u}})	&=\innp{\bfn}{\nabla\innp{\bfn}{\nabla \ol{u}}} + \innp{\bfn}{\nabla \ol{u}}(\nabla\cdot \bfn)\\
		&=\innp{\bfn}{\nabla\innp{\bfn}{\nabla \ol{u}}} -2H \innp{\bfn}{\nabla \ol{u}}.
		\end{align*}
		And,
		\begin{align*}
		\bfn^\top\nabla(\bfn\innp{\bfn}{\nabla \ol{u}})\bfn	&=\bfn^\top(\nabla\innp{\bfn}{\nabla \ol{u}}\bfn^\top + \innp{\bfn}{\nabla \ol{u}}\nabla \bfn)\bfn\\
		&=\bfn^\top\nabla\innp{\bfn}{\nabla \ol{u}}\bfn^\top\bfn + \bfn^\top\innp{\bfn}{\nabla \ol{u}}\nabla \bfn\bfn\\
		&=\bfn^\top\nabla\innp{\bfn}{\nabla \ol{u}} + \innp{\bfn}{\nabla \ol{u}}\bfn^\top\nabla \bfn\bfn\\
		&=\innp{\bfn}{\nabla\innp{\bfn}{\nabla \ol{u}}}.
		\end{align*}
		Here, we use the fact that $ \bfn^\top\bfn=1 $ and $ \bfn^\top\nabla \bfn\bfn=0 $. The second equality is because the symmetric curvature tensor $ \nabla\bfn $ is generated by the basis of $ \Ga $ hence is orthogonal to $ \bfn $. Combining these together, we have
		\begin{align*}
		\delg u - g	&=\Delta \ol{u} -\nabla\cdot(\bfn\innp{\bfn}{\nabla \ol{u}})-\bfn^\top \nabla^2 \ol{u} \bfn+\bfn^\top \nabla(\bfn\innp{\bfn}{\nabla \ol{u}})\bfn-\ol{g}\\
		&=\Delta \ol{u} - \ol{g} - \bfn^\top\nabla^2\ol{u}\bfn - \innp{\bfn}{\nabla\innp{\bfn}{\nabla \ol{u}}} +2H \innp{\bfn}{\nabla \ol{u}}+\innp{\bfn}{\nabla\innp{\bfn}{\nabla \ol{u}}}\\
		&=\Delta \ol{u} - \ol{g} - \bfn^\top\nabla^2\ol{u}\bfn+2H \innp{\bfn}{\nabla \ol{u}}.
		\end{align*}
		Therefore, 
		\begin{align*}
		\|\delg u - g\|	&\leq \|\Delta \ol{u} - \ol{g}\|+\|\bfn^\top\nabla^2\ol{u}\bfn\|+2|H|\|\innp{\bfn}{\nabla \ol{u}}\|\\
		&\leq \max(1,2|H|)(\|\Delta \ol{u} - \ol{g}\|+\|\bfn^\top\nabla^2\ol{u}\bfn\|+\|\innp{\bfn}{\nabla \ol{u}}\|).
		\end{align*}
		This proves (\ref{est_del})
	\end{enumerate}
\end{proof}
\begin{remark}
	We point out that there is a more general result about (\ref{est_del}) has been proved at \cite{cheung2018surfacepde} (Theorem 2.1 and Corollary 2.2). Our result presents above provide a much easier proof in $ \R^3 $, which can be extended to high order derivative cases.
	
	Additionally, if the $ \bfv $ is a constant vector along normal, we have $ \nag\cdot\bfv = \nabla\cdot \bfv $, which is the equivalent principle. However, if we only have $ \bfv\cdot\bfn =0 $, then 
	\[
	\nag \cdot\bfv =\frac{1}{|\nabla\phi|}\nabla\cdot(|\nabla\phi|\bfv)\mbox{ on }\Ga,
	\]
	where $ \Ga=\{\bfx\in\R^3|\phi(\bfx)=0 \} $. This result has been proved in \cite{deck2010hnarrow}.
\end{remark}

\section{PINNs for time-evolving PDEs on surfaces}\label{sec:pinn}
As shown in \cite{pinn,raissi2018}, the PINN can solve a large class of PDEs. In
this section, we first recap the PINN method for PDEs. Then, we propose a new 
algorithm to solve time-dependent PDEs on the surface.

\subsection{PINNs for classic PDEs}
Let $ \calN[\cdot] $ be a general differential operator defined on $
\Omega\subset \R^d $, and $ \Ga_b $ is a union of (inner or outer) boundaries of
$ \Omega $. We consider the following PDE on $ \Omega $:
\begin{equation}\label{pde1}
\begin{cases}
\calN[u(\bfx)]	&=f(\bfx)\quad\mbox{ in }\Omega,\\
u(\bfx)	&=u_b(\bfx)\quad\mbox{ on }\Ga_b,
\end{cases}
\end{equation}
where $ u(\bfx) $ is the solution of PDE, $ \bfx=(x_1,x_2,\cdots,x_d) $ is the
variable of $ u $, and $ f $ and $ u_b $ are the given right-hand side term and
the boundary condition, respectively. %We aim to solve Eq.~\eqref{pde1} by using PINNs.
As suggested in \cite{pinn,raissi2018}, we use a fully connected neural network
(NN) with $ N_l $ hidden layers to solve Eq.~\eqref{pde1}. Each hidden layer 
consists of $ N_e(l) $ neurons, where $l$ is the index of the hidden layer.
The input layer consists of $ d $ neurons $ x_1, x_2, \cdots, x_d $
which are the variables of the solution $ u $, and the output layer consists of
one neuron $ u $.

After establishing an NN with input $ \bfx=(x_1,x_2,\cdots,x_d) $ and output
$ u $, we need to set up a loss function to train it. Here, we denote the 
prediction of the NN by $ u_h $. Since we want $ u_h $ to satisfy 
Eq.~\eqref{pde1}, the loss function can be defined by the following mean 
square error (MSE):
\begin{align*}
MSE	&=MSE_u+MSE_b\\
&=\frac{1}{N_u}\sum_{i=1}^{N_u}|\calN[u_h(\bfx_u^i)]-f_u^i|^2+\frac{1}{N_b}\sum_{i=1}^{N_b}|u_b^i-u_h(\bfx_b^i)|^2.
\end{align*}
Here, $ \{\bfx_u^i,f_u^i\} $ are collocation points in $\Omega$,
%for the PDE restriction in (\ref{pde1}), 
$ \{\bfx_b^i,u_b^i\} $ are training data on $\Gamma_b$, 
%for the boundary condition in (\ref{pde1}). 
and $ N_u $ and $ N_b $ are the sizes of these points, respectively. 
The term $ \calN[u_h(\bfx_u^i)] $ can be computed through the
automatic differentiation technique \cite{ad}. In Tensorflow, this can be
performed using the function \texttt{tf.gradient}. Ideally, $ MSE=0 $, which
indicates taht $ u_h $ satisfies Eq.~\eqref{pde1} exactly. Given these two sets of
collocation points, we may train our NN for $ u $ and then get the prediction for
each point in $ \Omega $, 
%This solves our PDE problem (\ref{pde1}). 
and this NN is called PINN. 

\subsection{The loss function for surface PDEs}
In this subsection, we will use Theorem \ref{thm} to refine the loss function we have used in \cite{fang2019spde}.
We begin with the time-independent problem $ \delg u = f $. In \cite{fang2019spde}, the loss function is given by
\begin{equation}
MSE=\frac{1}{N_u}\sum_{i=1}^{N_u}|\Delta u_h(\bfx_u^i)-f_u^i|^2+\frac{1}{N_u}\sum_{i=1}^{N_u}|\innp{\nabla u_h(\bfx_u^i)}{\bfn(\bfx_u^i)}|^2,\label{loss1}
\end{equation}
where $ \{\bfx_u^i\}_{i=1}^{N_u} $ are the training points on $ \Ga $, and $ u_h
$ is the solution by the PINN. However, Theorem \ref{thm} shows that this is not a
good choice as we discussed in Section 2.
%Theoretically, when $ MSE=0 $, we have $ \innp{\bfn}{\nabla u}=0 $ and $ \Delta u - f = 0 $ on training points. Therefore, on these points, the residue of the equation $ \delg u - f $ equals zero, which indicates that the equation is satisfied on the training points. But in practice computation, the $ MSE $ will never be zero. Consequently, the residue is not zero. 
Instead, Theorem \ref{thm} motivates us to set the loss function as 
\begin{equation}
MSE=\frac{1}{N_u}\sum_{i=1}^{N_u}|\Delta u_h(\bfx_u^i)-f_u^i|^2+\frac{1}{N_u}\sum_{i=1}^{N_u}|\innp{\nabla u_h(\bfx_u^i)}{\bfn(\bfx_u^i)}|^2+\frac{1}{N_u}\sum_{i=1}^{N_u}|\bfn(\bfx_u^i)^\top\nabla^2u_h\bfn(\bfx_u^i)|^2.\label{loss2}
\end{equation}
With this setting, the residue of the surface PDE approaches zero as $ MSE\to 0 $.
We note that, in practice, Eq.\eqref{loss1} works for most of the smooth surfaces, 
and Eq.\eqref{loss2} provides an indicator of PDEs' residue.

Now, we can handle the spatial derivative such as $ \nag $, $ \nag\cdot $, and $
\delg $. The next step is to consider the operator $ \pa_t $. 
We follow the continuous-time and discrete-time method for time-dependent PDE
in~\cite{pinn} but mofidy the spatial derivatives based on Theorem 2.1.

\subsection{Continuous-time PINN for surface PDE}
The most natural idea to solve the time-dependent problem is to deal with the time derivatives 
in the same manner as the spatial variable. Let us consider the following time-dependent surface PDE problem
\begin{eqnarray}
\pa_t u(\bfx,t) &=&\calN_\Ga[u(\bfx,t)]\mbox{ for }(\bfx,t)\in\Ga\times [0,T],\label{eqn2}\\
u(\bfx,0)	&=&u_0(\bfx)	\mbox{ for }\bfx\in\Ga,\label{ini}
\end{eqnarray}
where $ \Ga\subset \R^3 $ is a closed surface, and $ \calN_\Ga $ is a spatial
differential operator on $ \Ga $, consisting of $ \delg $, $ \nag\cdot $ and $
\nag $. Although this assumption on $ \calN_\Ga $ does not include all cases
of the surface differential operators, it covers the most scenarios we focus on.
For more complicated operators, such as $ \delg^2 $, we may develop a similar 
result, as shown in Theorem~\ref{thm}, and then establish the corresponding loss
function accordingly. The key point is that the loss function should be an
indicator for the residue of the surface PDE.

Base on the Theorem~\ref{thm}, we define the following loss function for (\ref{eqn2}):
\begin{eqnarray}
MSE	&=&\frac{1}{N_u}\sum_{i=1}^{N_u}\left|\pa_t u_h(\bfx_u^i,t_u^i)-\calN_\Ga[u(\bfx_u^i,t_u^i)]\right|^2+\frac{1}{N_u}\sum_{i=1}^{N_u}\left|\innp{\nabla u_h(\bfx_u^i,t_u^i)}{\bfn(\bfx_u^i)}\right|^2\nonumber\\
&&+\frac{1}{N_u}\sum_{i=1}^{N_u}\left|\bfn(\bfx_u^i)^\top\nabla^2u(\bfx_u^i,t_u^i)\bfn(\bfx_u^i)\right|^2+\frac{1}{N_0}\sum_{i=1}^{N_0}\left|u_h(\bfx_0^i,0)-u_0(\bfx_0^i)\right|^2,\label{loss_sur}
\end{eqnarray}
where $ \{\bfx_u^i,t_u^i\}_{i=1}^{N_u} $ are collocation points for the PDE, 
$\{\bfx_0^i\}_{i=1}^{N_0} $ are collocation points for the initial data, and
$\bfn$ is the unit outer normal on $ \Ga $. According to Theorem~\ref{thm}, 
Eqs.~\eqref{eqn2} and \eqref{ini} are enforced by minimizing
Eq.\eqref{loss_sur}.

There are several methods to generate collocation points on $ \Ga\times [0,T] $.
In the numerical examples shown below, we used the following two methods. If the
surface is defined by implicit equations and homeomorphic to sphere, for
example, $ \frac{x^2}{2}+y^2+z^2=1 $, we may generate Fibonacci lattice on the
unit sphere, and then map them to the target surface. We then obtain collocation
points on $ \Ga $. Next, we choose a set of partition points on $ [0,T] $, and the
collocation points $ \{\bfx_u^i,t_u^i\}_{i=1}^{N_u} $ are given by the tensor 
product of the collocation points on $ \Ga $ and the partition points on $ [0,T] $. If the surface is given by the parametric form, we may use the Latin hypercube sampling (LHS) to generate a set of points $ (\alpha_i,\beta_i, t_i)_{i=1}^{N_u} $. Next, we map $ (\alpha_i,\beta_i) $ to the surface through the parametric equation of the surface. Then, we get the collocation points $ \{\bfx_u^i,t_u^i\}_{i=1}^{N_u} $.

We can predict the solution for any  $ (\bfx,t)\in\Ga\times [0,T] $ by training this PINN. But the training set is large because it includes all the data between $ [0,T] $. This issue can be fixed by the discrete-time method.

\subsection{Discrete-time PINN for surface PDE}
%Rather than just outputting a function value $ u $ with the input $ \bfx $ and $ t $,
Alternatively, the discrete-time method first discretize the time derivative by
using high order Runge-Kutta methods. Then the PINN provides a sequence of 
function evaluations at different time partition points.
%Let us reconsider the abstract PDE problem (\ref{eqn2}). 
We use Eq.~\eqref{eqn2} to describe the entire procedure.
We follow the notations in~\cite{pinn}, and the general form of $ q $ stages
Runge-Kutta scheme of (\ref{eqn2}) is given by
\begin{equation}
  \label{eq:rk}
\begin{aligned}
u^{n+c_i}	&=u^n-\Delta t\sum_{j=1}^qa_{ij}\calN[u^{n+c_j}],\qquad i=1,\cdots ,q,\\
u^{n+1}		&=u^n-\Delta t\sum_{j=1}^q b_j\calN[u^{n+c_j}].
\end{aligned}
\end{equation}
Here $ u^{n+c_j}(\bfx)=u(\bfx,t^n+c_j\Delta t) $, is the hidden state of the
system at time $ t^n+c_j\Delta t $ for $ j=1,\cdots,q $. In this problem, we
set $ t^n=0 $ and $ t^{n+1}=T $. 
We proceed by placing a multi-output neural network prior with input $ \bfx $ and output
\[
(u^{n+c_1},\cdots,u^{n+c_q},u^{n+1}).
\]
The theoretical error estimate implies that the temporal error of this method is
$ O(\Delta t^{2q}) $. Therefore, the prerequisite of this method is that $
\Delta t<1 $, otherwise the solution may not converge. 
For $ T\geq 1 $, we can set a reference time $ 0<\widetilde{T}<1 $, and apply the transform $ \widetilde{t}=\frac{t}{T}\widetilde{T} $. The resulting PDE is
\begin{equation}\label{eqn_trans}
\pa_{\widetilde{t}}u(\bfx,\widetilde{t})=\frac{T}{\widetilde{T}}\calN_\Ga[u(\bfx,\widetilde{t})].
\end{equation}
Then we can use the same method to handle it.

To design the loss function, we use the following notation:
%\begin{align}
%u^n		&=u_i^n,\qquad i=1,\cdots,q,\label{rk1}\\
%u^n		&=u^n_{q+1},\label{rk2}
%\end{align}
%where
\begin{align}
u_i^n	&=u^{n+c_i}+\Delta t\sum_{j=1}^qa_{ij}\calN[u^{n+c_j}],\qquad i=1,\cdots,q,\label{rk3}\\
u_{q+1}^n	&=u^{n+1}+\Delta t\sum_{j=1}^qb_j\calN[u^{n+c_j}].\label{rk4}
\end{align}
Therefore, given the collocation points $ \{\bfx_u^i,t_u^i\}_{i=1}^{N_u} $ as mentioned above, the loss function is given by
\[
  MSE=\frac{1}{N_u}\sum_{j=1}^{q+1}\sum_{i=1}^{N_u}\left(|u_j^n(\bfx^i)-u_0(\bfx^i)|^2+|\innp{\nabla u_j^n(\bfx^i)}{\bfn(\bfx^i)}|^2+|\bfn(\bfx^i)^\top\nabla^2u_j^n(\bfx^i)\bfn(\bfx^i)|^2\right).
\]
%where $ u_j^n $ is given by (\ref{rk1})-(\ref{rk4}). 
By training this PINN, we will get the solution of (\ref{eqn2})-(\ref{ini}) not only at the final time but also some intermediate time points. In this case, we cannot get the predicted solution directly for any given time $ t\in[0,T] $ other than the Runge Kutta nodes. But we can use the interpolation method to recover it. Compared with the continuous-time method, this method saves time because we have fewer training points than the continuous-time method.

\section{Numerical Experiments}\label{sec:num}
In this section, we present numerical results using continuous- and discrete-time
approaches to demonstrate the efficacy of our PINN method for surface PDE. All
numerical experiments were conducted by using Tensorflow 1.13 on DGX-2 with 
NVIDIA Tesla V100. For the continuous-time method, we set up a PINN of $4$ 
hidden layers with $100$ neurons each layer, and $1$ neuron for the output 
layer. For the discrete-time method, we set up a PINN of $4$ hidden layers, 
with $200$ neurons each layer and $q+1$ neurons for the output layer, where $q$
stages Runge-Kutta method has been used. We choose $\sigma(s)=\sin(\pi s)$ as
the activation function~\cite{fang2019spde,fang2019meta}. 

\subsection{Continuous-Time Method}
Consider the following heat equation
%-------------------------------------------------------------------------------
\begin{eqnarray}
\pa_tu(\bfx,t)	&=&	\delg u(\bfx,t) + f(\bfx,t)\mbox{ for }(\bfx,t)\in\Ga\times[0,T],\label{num_eqn}\\
u(\bfx,0)	&=&	u_0(\bfx)		\mbox{ for } \bfx\in\Ga.\label{num_ini}
\end{eqnarray}
%-------------------------------------------------------------------------------
We first test the accuracy of the continuous-time method. To this end, we set $
u(\bfx,t)=x_1\sin(tx_2)+x_3 $ as the exact solution, where $ \bfx=(x_1,x_2,x_3)
$ and $ \Ga=\S^2 $ is the unit sphere. We choose a suitable $ f $ to make the
PDE hold, and set $ u_0(\bfx)=u(\bfx,0) $. To evaluate the performance of our
algorithm, we define the relative error of prediction $ u_h $ by PINNs. Let
$ u $ be the exact solution, and $ \{\bfx_c^i\}_{i=1}^{N_c} $ be a set of sample
points for the accurcy test, which differs from training set 
$ \{\bfx_u^i\}_{i=1}^{N_u} $. Then, the relative error $Err$ is given by
\[
Err = \frac{\sqrt{\sum_{i=1}^{N_c}|u_h(\bfx_c^i)-u(\bfx_c^i)|^2}}{\sqrt{\sum_{i=1}^{N_c}|u(\bfx_c^i)|^2}}.
\]
We generate $ 50,000 $ collocation points on $ \Ga\times [0,1] $ by using the tensor product of Fibonacci lattice, as mentioned above. 

Figure~\ref{fig:con_error} shows the results at different time. We can see that 
the relative error is around $ 10^{-2} $. Of note, in this example, if
we drop out the $\bfn\nabla^2u\bfn$ term in the loss function, the PINN will converge
slightly faster. This is probably because $ \Ga $ is a very smooth surface with
a constant curvature. This also enlightens us to adjust the weight of this second
order term in the loss function, which is currently $ 1 $. However, the pattern of 
how this weight affects the convergence is unclear at this time and will remain 
as future work.
%-------------------------------------------------------------------------------
\begin{figure}[thbp]
	\centering
  \subfigure[Time: $0.25$, $Err=3.098151e-02$]{
			\includegraphics[width=0.48\textwidth]{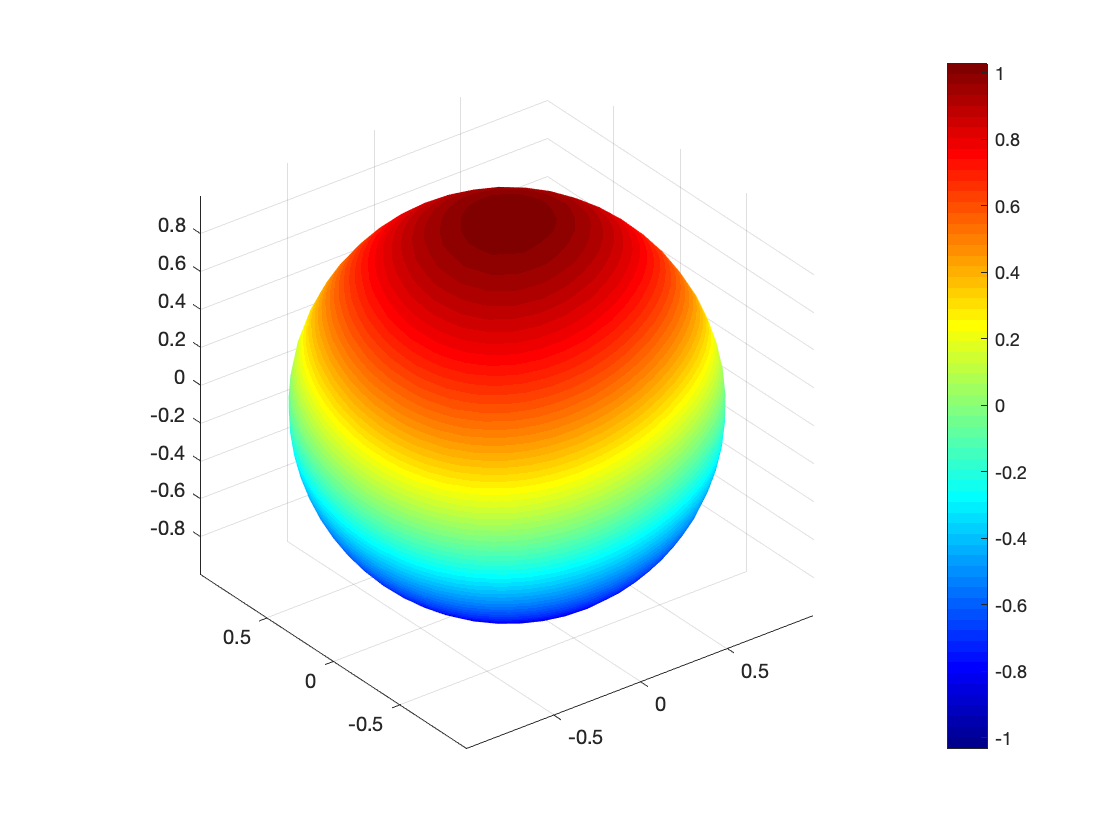}}
  \subfigure[Time: $0.5$, $Err=5.168469e-02$]{
			\includegraphics[width=0.48\textwidth]{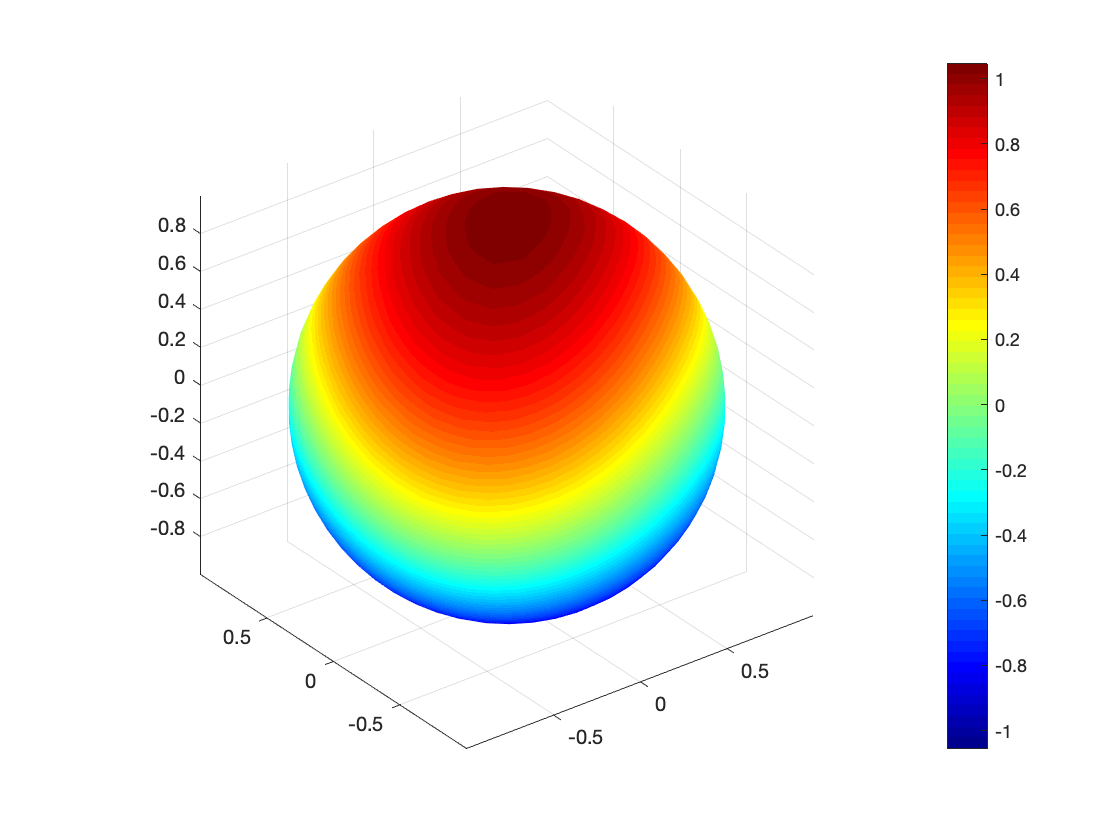}} \\
  \subfigure[Time: $0.75$, $ Err=6.373482e-02$]{
			\includegraphics[width=0.48\textwidth]{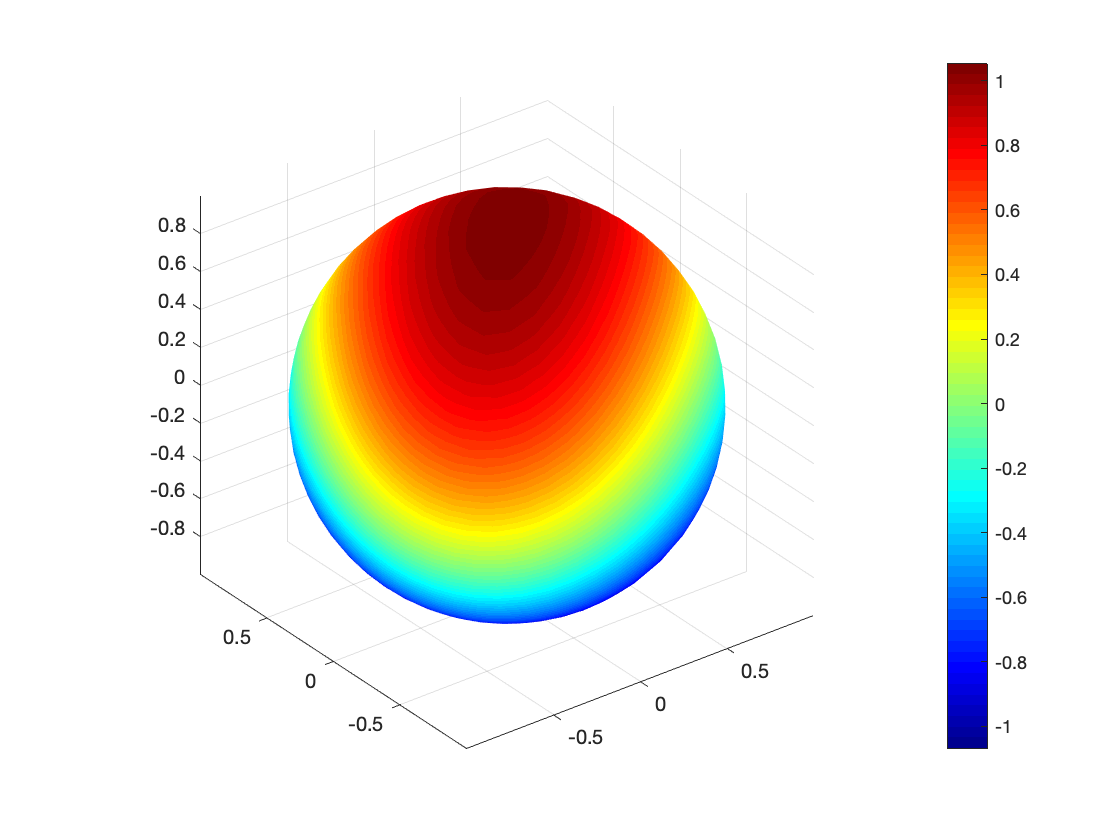}}
  \subfigure[Time: $ 1.0 $, $ Err=6.761617e-02$]{
      \includegraphics[width=0.48\textwidth]{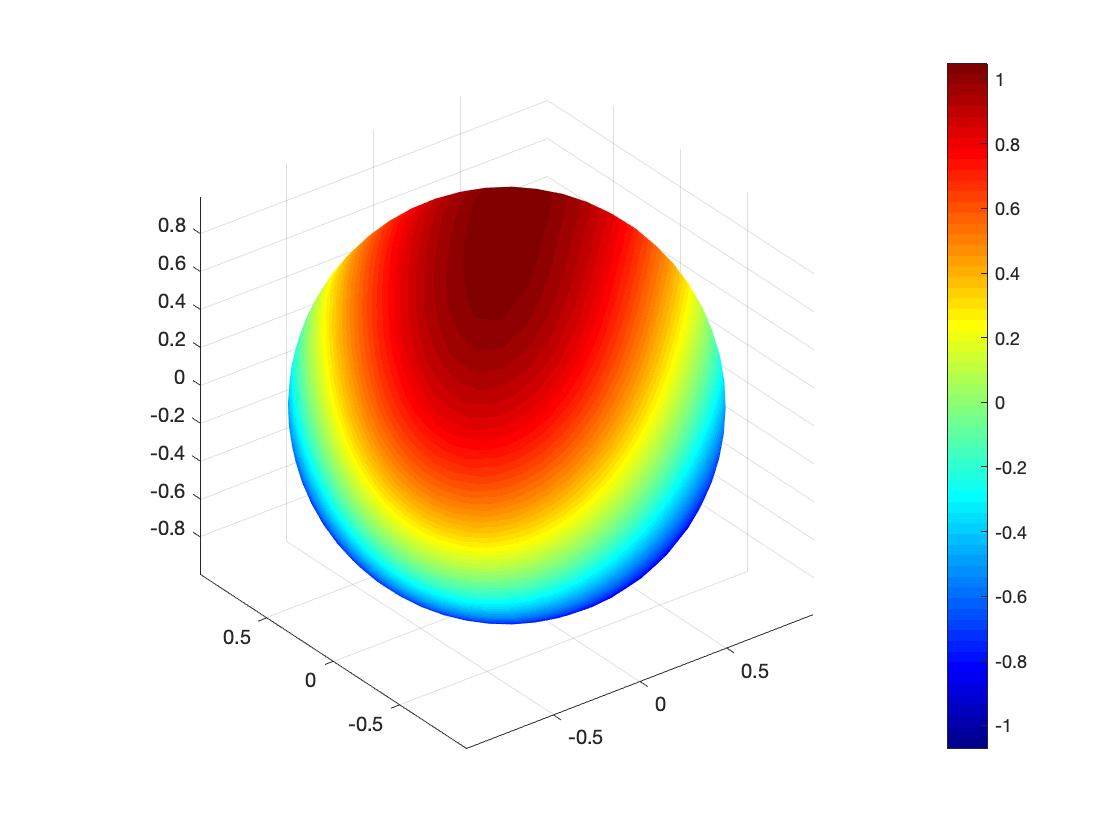}}
	\caption{Error of continuous-time method for different time steps.}
	\label{fig:con_error}
\end{figure}
%-------------------------------------------------------------------------------

Next, we apply our method to a benchmark problem named heating a
torus~\cite{dziuk2013finite}. Let 
\[
\Ga = \left\{\bfx\in\R^3\bigg|\left(\sqrt{x_1^2+x_2^2}-1\right)^2+x_3^2=\frac{1}{16}\right\},
\]
with the right-hand side being a regularized version of the characteristic function 
\[
f(\bfx,t)=100\chi_G(\bfx),\qquad \bfx\in\Ga,
\]
where $x\in\Ga||\bfx-(0,1,0)|\leq 0.25\} $ and $ \chi_G $ is the characteristic 
function on $ G $. The initial condition is set as $u=0$. 
%We solve Eqs.~\eqref{num_eqn}-\eqref{num_ini} by using these data. 
We generate $50,000$ collocation points by using LHS as mentioned above. The 
results are shown in Figure \ref{fig:ht}.
%-------------------------------------------------------------------------------
\begin{figure}[thbp]
	\centering
  \subfigure[Time: $0.0$]{
			\includegraphics[width=0.48\textwidth]{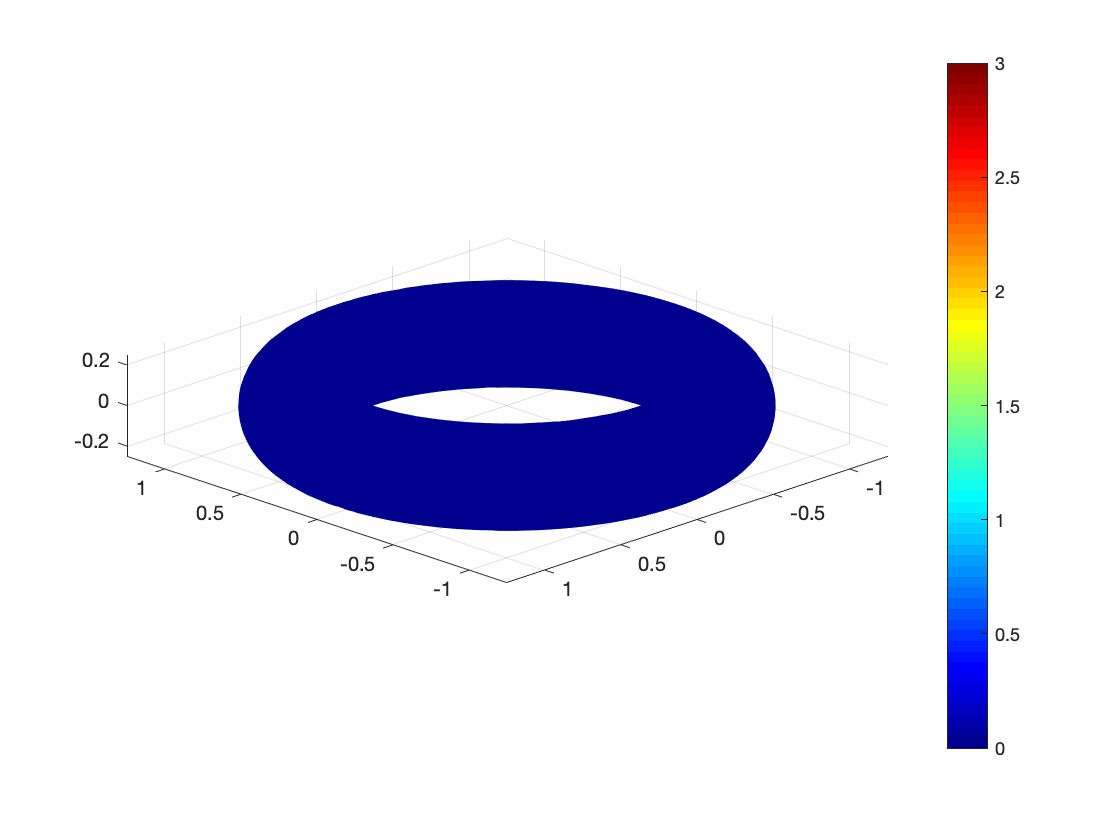} }
  \subfigure[Time: $0.75$]{
			\includegraphics[width=0.48\textwidth]{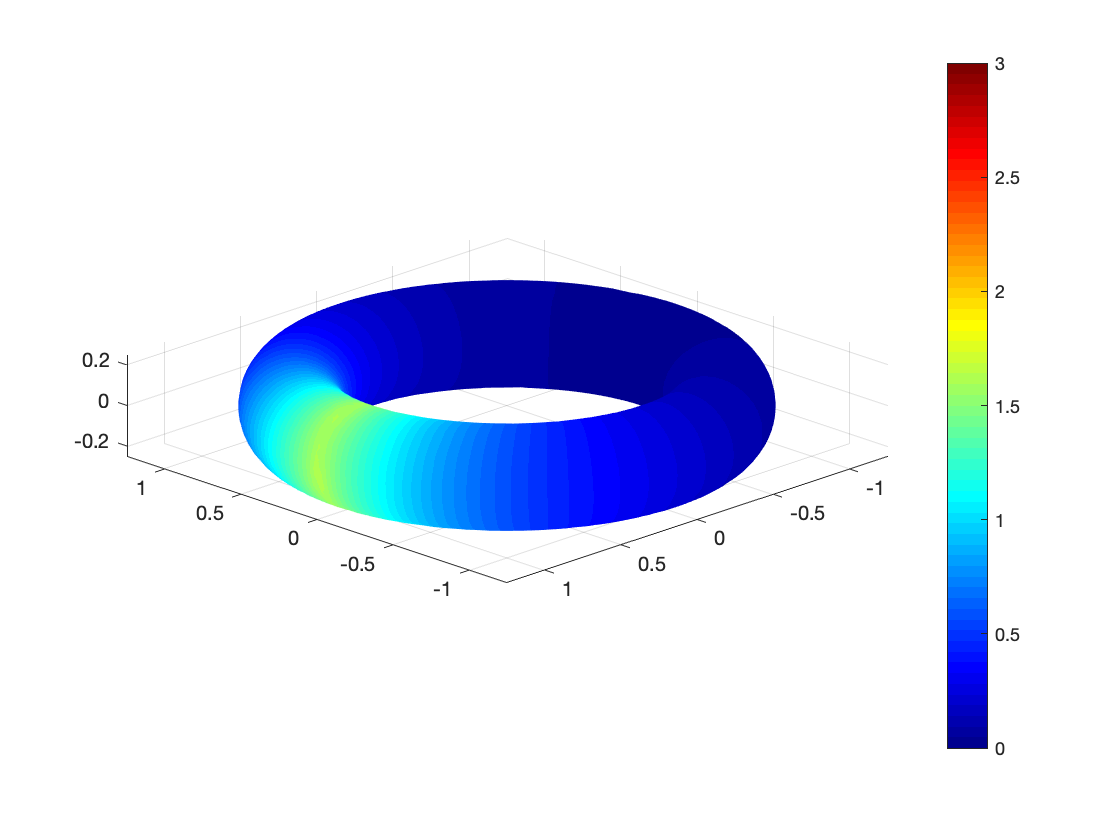} } \\
  \subfigure[Time: $1.5$]{
			\includegraphics[width=0.48\textwidth]{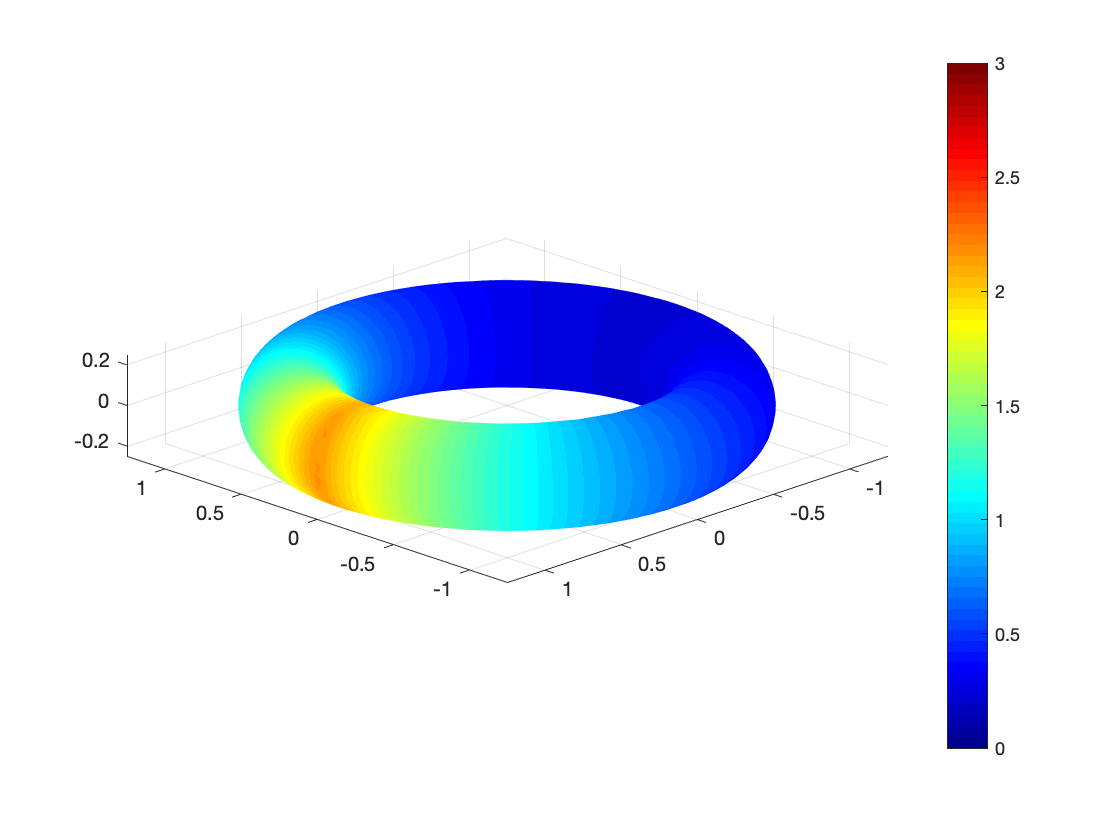} }
  \subfigure[Time: $2.25$]{
			\includegraphics[width=0.48\textwidth]{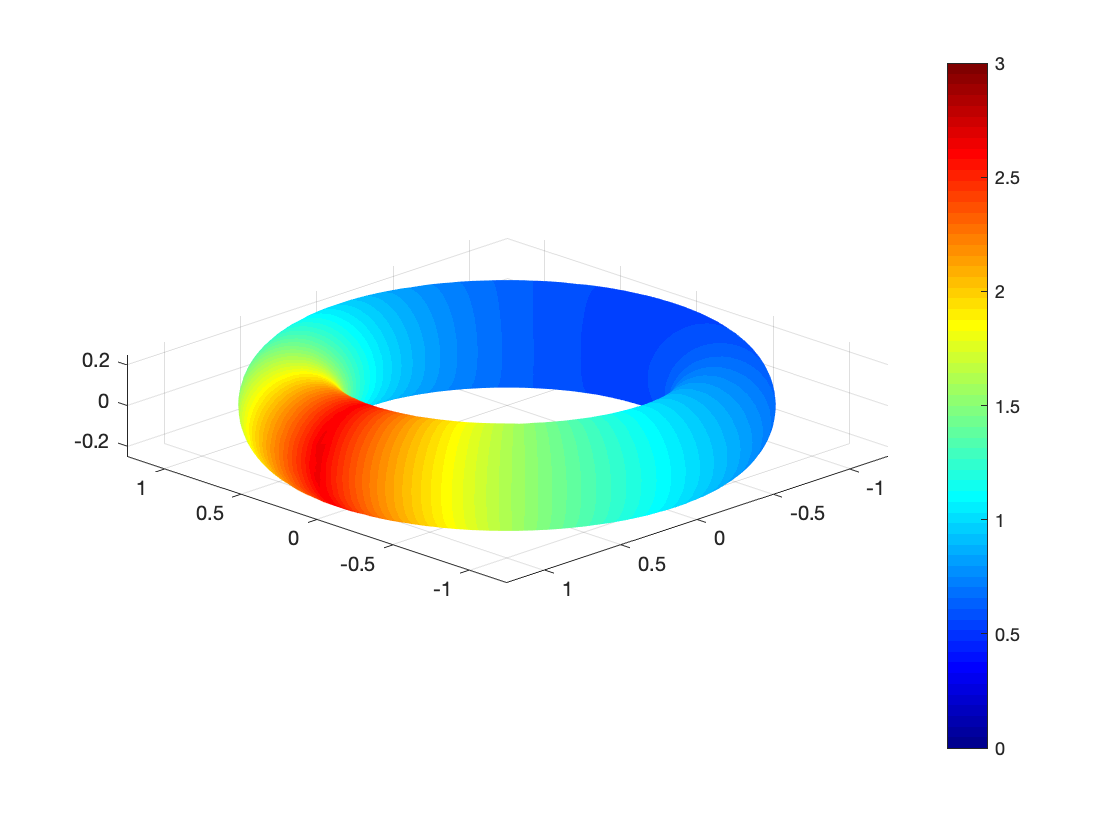} } \\
  \subfigure[Time: $3.0$]{
		\includegraphics[width=0.48\textwidth]{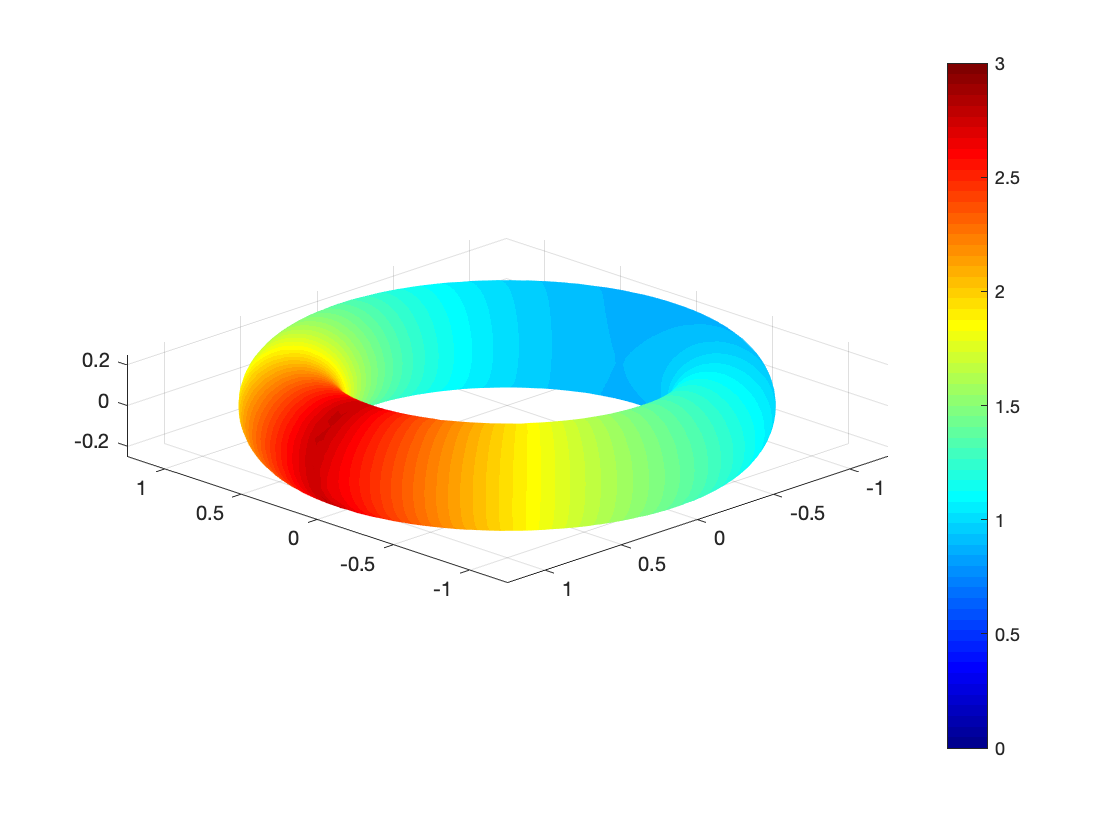} }
	\caption{Heating a torus. The color indicates the magnitude of the solution. Viewpoint: $ (-1,1,0.5) $.}\label{fig:ht}
\end{figure}
%-------------------------------------------------------------------------------

\subsection{Discrete-Time Method}
In this subsection, we test the discrete-time method. We first set $ 0.5=T<1 $
and reconsider (\ref{num_eqn})-(\ref{num_ini}) on $ \Ga=\S^2 $. For the
discrete-time method, we only need collocation points on $ \Ga $ as the time
derivative has already been discretized by the Runge-Kutta method. Therefore, 
$500$ Fibonacci lattices on $ \S^2 $ has been generated as the collocation
points. This number is $ 100 $ times smaller than that of the continuous-time 
method. Thus, it saves memory and computational time significantly. A $ 100 $ 
stages Runge-Kutta method has been used. The theoretical result indicates that 
the time error of this $0.5^{200}\approx 6.223015278\times 10^{-61} $, which is
much smaller than machine precision. We set the exact solution as $ 
u(\bfx,t)=x_1x_2x_3\exp(t) $.

The result of $ u(\bfx,T) $ is shown in Figure \ref{fig:dis05}. As we can see,
the relative error is of the same order as the continuous-time method, because 
in this case, the spatial PINN error dominates the total error.

\begin{figure}
	\centering
	\includegraphics[width=0.5\textwidth]{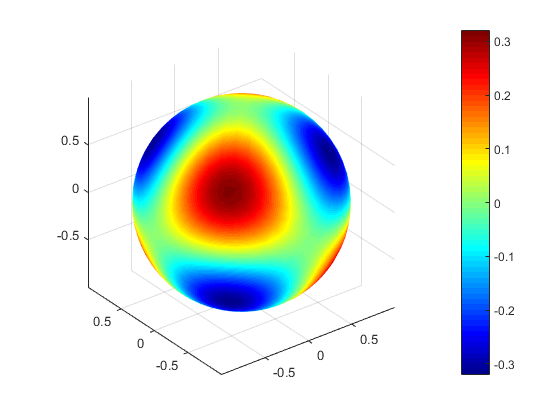}
	\caption{Discrete-time method with $ T=0.5$. $ Err=1.467495e-02 $.}
	\label{fig:dis05}
\end{figure}

Next, we set $ 3=T>1 $ to verify this method works after the time rescaling. We choose $ \widetilde{T}=0.5 $ and transform $ [0,T] $ to $ [0,\widetilde{T}] $, then we apply our discrete-time method to (\ref{eqn_trans}). We set our exact solution $ u(\bfx,t) $ as $ x_1x_2x_3\exp(t) $ and $ x_1\sin(tx_2)+x_3 $, respectively. All the other settings keep the same. The result of $ u(\bfx,T) $ is shown in Figure \ref{fig:dis30}, whose relative error is still around $ 10^{-2} $.

\begin{figure}
	\centering
  \subfigure[$u(\bfx,t)=x_1x_2x_3\exp(t)$. $Err=5.967653e-02$]{
			\includegraphics[width=0.48\textwidth]{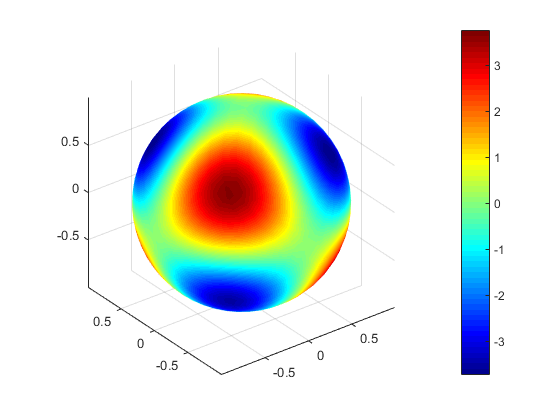} }
  \subfigure[$u(\bfx,t)=x_1\sin(tx_2)+x_3$. $Err=8.814725e-02$]{
			\includegraphics[width=0.48\textwidth]{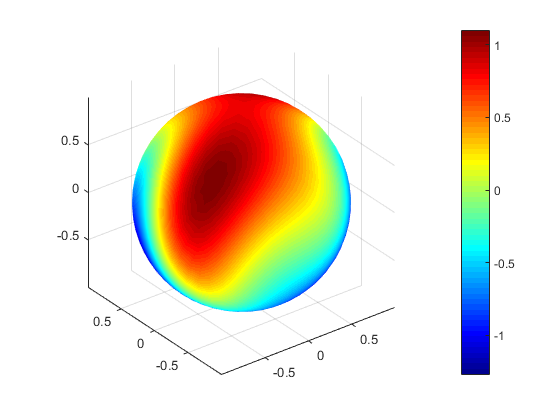}}
	\caption{Discrete-time method with $ T=3 $.}
	\label{fig:dis30}
\end{figure}
%{\color{red}Since this paper is a proceeding work of \cite{fang2019spde}, we by the way mention another typo in \cite{fang2019spde}: on page 26331, $ f=12x_1x_2x_3 $ instead of $ f=18x_1x_2x_3 $, and also in page 26332, the eigenvalue is $ 12 $ instead of $ 18 $.}
\section{Conclusion}\label{sec:con}
In this paper, we put forward the PINN method for time-evolving PDEs on
surfaces. A simplified proof of the prior estimate of the differential operator
on the surface has been shown in theorem \ref{thm}. To deal with the temporal
differential operator, we come up with the continuous-time method and
discrete-time method. The continuous-time method is more flexible and
straightforward. Although, in general, there is no restriction on the 
length of the time interval. But this method will take more sample points and
then more time to train the PINN.

On the other hand, the discrete-time method needs a relatively small collocation
points set and, hence, less training time. But it requires the time interval less 
than $1$ due to the error analysis of the Range Kutta scheme. The numerical 
experiments verified our algorithm.

%As future work, the PINN for time-evolving surface PDEs on the moving surface is a promising topic to investigate.

\clearpage
\renewcommand{\baselinestretch}{0.8}


\begin{thebibliography}{99}
	% avoid the  blank line in between the references
	\setlength{\parskip}{-5pt}
		\bibitem{turk1991}Turk, G., 1991, July. \emph{Generating textures on arbitrary surfaces using reaction-diffusion}. In ACM SIGGRAPH Computer Graphics (Vol. 25, No. 4, pp. 289-298). ACM.
	
	\bibitem{bertozzi2001}Bertalmio, M., Bertozzi, A.L. and Sapiro, G., 2001, December. \emph{Navier-Stokes, fluid dynamics, and image and video inpainting}. In Proceedings of the 2001 IEEE Computer Society Conference on Computer Vision and Pattern Recognition. CVPR 2001 (Vol. 1, pp. I-I). IEEE.
	
	\bibitem{tian2009}Tian, L., Macdonald, C.B. and Ruuth, S.J., 2009, November. \emph{Segmentation on surfaces with the closest point method}. In 2009 16th IEEE International Conference on Image Processing (ICIP) (pp. 3009-3012). IEEE.
	
	\bibitem{biddle2013}Biddle, H., von Glehn, I., Macdonald, C.B. and M\"{a}rz, T., 2013, September. \emph{A volume-based method for denoising on curved surfaces}. In 2013 IEEE International Conference on Image Processing (pp. 529-533). IEEE.
	
	\bibitem{murray2003}Murray, J.D., 2003. II. \emph{Spatial Models and Biomedical Applications}. Springer.
	
	\bibitem{olsen1998}Olsen, L., Maini, P.K. and Sherratt, J.A., 1998. \emph{Spatially varying equilibria of mechanical models: Application to dermal wound contraction}. Mathematical biosciences, 147(1), pp.113-129.
	
	\bibitem{toga1998}Toga, A.W. ed., 1998. \emph{Brain warping}. Elsevier.
	
	\bibitem{elliott2010biomem}Elliott, C.M. and Stinner, B., 2010. \emph{Modeling and computation of two phase geometric biomembranes using surface finite elements}. Journal of Computational Physics, 229(18), pp.6585-6612.
	
	\bibitem{halpern1998}Halpern, D., Jensen, O.E. and Grotberg, J.B., 1998. \emph{A theoretical study of surfactant and liquid delivery into the lung}. Journal of Applied Physiology, 85(1), pp.333-352.
	
	\bibitem{dziuk1988fem}Dziuk, G., 1988. \emph{Finite elements for the Beltrami operator on arbitrary surfaces}. In Partial differential equations and calculus of variations (pp. 142-155). Springer, Berlin, Heidelberg.
	
	\bibitem{bertalmio2000}Bertalmio, M., Cheng, L.T., Osher, S. and Guillermo, S., 2000. \emph{Variational problems and partial differential equations on implicit surfaces: The framework and examples in image processing and pattern formation}.
	
	\bibitem{cheung2018surfacepde}Cheung, K.C. and Ling, L., 2018. A kernel-based embedding method and convergence analysis for surfaces PDEs. SIAM Journal on Scientific Computing, 40(1), pp.A266-A287.
	
	\bibitem{yang2019gan}Yang, Y. and Perdikaris, P., 2019. Adversarial uncertainty quantification in physics-informed neural networks. Journal of Computational Physics, 394, pp.136-152.
	
	\bibitem{fang2019spde}Fang, Z. and Zhan, J., 2019. A Physics-Informed Neural Network Framework For Partial Differential Equations on 3D Surfaces: Time Independent Problems. IEEE Access.
	
	\bibitem{fang2019meta}Fang, Z. and Zhan, J., 2019. Deep Physical Informed Neural Networks for Metamaterial Design. IEEE Access, 8, pp.24506-24513.
	
	\bibitem{vpinn}Kharazmi, E.H.S.A.N., Zhang, Z. and Karniadakis, G.E., 2019. Variational Physics-Informed Neural Networks For Solving Partial Differential Equations. arXiv preprint arXiv:1912.00873.
	
	\bibitem{hpvpinn}Kharazmi, E., Zhang, Z. and Karniadakis, G.E., 2020. hp-VPINNs: Variational Physics-Informed Neural Networks With Domain Decomposition. arXiv preprint arXiv:2003.05385.
	
	\bibitem{dziuk2013finite}Dziuk, G. and Elliott, C.M., 2013. Finite element methods for surface PDEs. Acta Numerica, 22, pp.289-396.
	
	\bibitem{petras2016}Petras, A. and Ruuth, S.J., 2016. \emph{PDEs on moving surfaces via the closest point method and a modified grid based particle method}. Journal of Computational Physics, 312, pp.139-156.
	
	\bibitem{petras2018rbffd}Petras, A., Ling, L. and Ruuth, S.J., 2018. An RBF-FD closest point method for solving PDEs on surfaces. Journal of Computational Physics, 370, pp.43-57.
	
	\bibitem{deck2010hnarrow}Deckelnick, K., Dziuk, G., Elliott, C.M. and Heine, C.J., 2010. An h-narrow band finite-element method for elliptic equations on implicit surfaces. IMA Journal of Numerical Analysis, 30(2), pp.351-376.
	
	\bibitem{petras2019}Petras, A., Ling, L., Piret, C. and Ruuth, S.J., 2019. A least-squares implicit RBF-FD closest point method and applications to PDEs on moving surfaces. Journal of Computational Physics, 381, pp.146-161.
	
	\bibitem{pinn}Raissi, M., Perdikaris, P. and Karniadakis, G.E., 2019. Physics-informed neural networks: A deep learning framework for solving forward and inverse problems involving nonlinear partial differential equations. Journal of Computational Physics, 378, pp.686-707.
	
	\bibitem{raissi2018}Raissi, M., 2018. Deep hidden physics models: Deep learning of nonlinear partial differential equations. The Journal of Machine Learning Research, 19(1), pp.932-955.
	
	\bibitem{xavier}Glorot, X. and Bengio, Y., 2010, March. Understanding the difficulty of training deep feedforward neural networks. In Proceedings of the thirteenth international conference on artificial intelligence and statistics (pp. 249-256).
	
	\bibitem{ad}Baydin, A.G., Pearlmutter, B.A., Radul, A.A. and Siskind, J.M., 2018. Automatic differentiation in machine learning: a survey. Journal of Marchine Learning Research, 18, pp.1-43.
	
	\bibitem{goodfellow2014}Goodfellow, I.J., Shlens, J. and Szegedy, C., 2014. Explaining and harnessing adversarial examples. arXiv preprint arXiv:1412.6572.
	
	\bibitem{rastegari2016}Rastegari, M., Ordonez, V., Redmon, J. and Farhadi, A., 2016, October. Xnor-net: Imagenet classification using binary convolutional neural networks. In European Conference on Computer Vision (pp. 525-542). Springer, Cham.
	
	\bibitem{gawehn2016}Gawehn, E., Hiss, J.A. and Schneider, G., 2016. Deep learning in drug discovery. Molecular Informatics, 35 (1), 3–14.
	
	\bibitem{alipanahi2015}Alipanahi, B., Delong, A., Weirauch, M.T. and Frey, B.J., 2015. Predicting the sequence specificities of DNA-and RNA-binding proteins by deep learning. Nature biotechnology, 33(8), p.831.

	\bibitem{nvidia} NVDIA SimNET \texttt{https://developer.nvidia.com/simnet}

\end{thebibliography}
\end{document}